\documentclass[sts]{imsart}

\RequirePackage[OT1]{fontenc}
\usepackage{amsthm,amsmath,natbib}
\RequirePackage[colorlinks,citecolor=blue,urlcolor=blue]{hyperref}
\usepackage[dvipsnames]{xcolor}
\usepackage{amssymb}
\usepackage{graphicx}
\usepackage{fullpage}

\startlocaldefs
\numberwithin{equation}{section}
\theoremstyle{plain}
\newtheorem{thm}{Theorem}[section]
\newtheorem{lem}{Lemma}[section]
\newtheorem{coro}{Corollary}[section]
\newtheorem{defn}{Definition}
\newtheorem{remark}{Remark}[section]
\endlocaldefs

\DeclareMathOperator*{\argmin}{arg\,min}

\DeclareMathOperator*{\E}{\mathbb{E}}
\newcommand{\F}{\mathcal{F}}

\begin{document}

\begin{frontmatter}
\title{How Well Can Generative Adversarial Networks Learn Densities: A Nonparametric View}
\runtitle{Nonparametric Estimation with GAN}

\begin{aug}
\author{\fnms{Tengyuan} \snm{Liang}\thanksref{t1}\ead[label=e1]{tengyuan.liang@chicagobooth.edu}}

\thankstext{t1}{\printead{e1}}

\affiliation{University of Chicago}

%
\end{aug}

\begin{abstract}
We study in this paper the rate of convergence for learning densities under the Generative Adversarial Networks (GAN) framework, borrowing insights from nonparametric statistics. We introduce an improved GAN estimator that achieves a faster rate, through simultaneously leveraging the level of smoothness in the target density and the evaluation metric, which in theory remedies the mode collapse problem reported in the literature. A minimax lower bound is constructed to show that when the dimension is large, the exponent in the rate for the new GAN estimator is near optimal. One can view our results as answering in a quantitative way how well GAN learns a wide range of densities with different smoothness properties, under a hierarchy of evaluation metrics. As a byproduct, we also obtain improved generalization bounds for GAN with deeper ReLU discriminator network. 
\end{abstract}

\begin{keyword}
\kwd{Generative adversarial networks}
\kwd{nonparametric statistics}
\kwd{density estimation}
\kwd{Sobolev spaces}
\kwd{oracle inequality}
\kwd{neural network learning}
\end{keyword}

\end{frontmatter}

\section{Introduction}

Generative Adversarial Networks (GANs) \citep{goodfellow2014generative, li2015generative, arjovsky2017wasserstein, dziugaite2015training} have stood out as an important unsupervised method for learning and efficient sampling from a complicated, multi-modal target data distribution. Despite its celebrated empirical success in image tasks, there are many theoretical questions to be answered \citep{liu2017approximation,arora2017gans,arora2017generalization}.

One convenient formulation of the GAN framework \citep{arjovsky2017wasserstein,li2015generative, dziugaite2015training, liu2017approximation} solves the following minimax problem, at the population level,
\begin{align*}
	\min_{\mu \sim \mu_{G}} \max_{f \in \F_{D}} \E_{Y \sim \mu} f(Y) - \E_{X \sim \nu} f(X).
\end{align*}
In plain language, given a target distribution $\nu$, one seeks for a distribution $\mu$ from a probability distribution \textit{generator class} $\mu_G$, such that it minimizes the loss incurred by the best discriminator function inside the \textit{discriminator class} $\F_D$. In practice, both the \textit{generator class} and the \textit{discriminator class} are represented by neural networks. $\mu_G$ quantifies the transformed distributions realized by a network with random inputs (either Gaussian or uniform distribution), and $\F_D$ represents the functions that are realizable by a certain neural network architecture. We refer the readers to \cite{liu2017approximation} for other more general formulations of GANs.

In practice, one only has access to finite samples of the real data distribution $\nu$.
Denote $\tilde{\nu}_n$ to be a measure estimate based on $n$ i.i.d. samples from $\nu$, where the empirical density $d \hat{\nu}_n(x)/dx = \frac{1}{n} \sum_{i=1}^n \delta_{X_i}(x)$ is typically used. Given the samples, the GAN solves the following problem
\begin{align}
	\label{eq:gan}
	\tilde{\mu}_n = \argmin_{\mu \sim \mu_{G}} \max_{f \in \F_{D}} \E_{Y \sim \mu} f(Y) - \E_{X \sim \tilde{\nu}_n} f(X).
\end{align}
Two fundamental yet basic questions that puzzle machine learning theorists are: (1) how well does GAN learn the densities statistically, overlooking the optimization difficulty? (2) how well does the iterative optimization dynamics of solving the minimax problem approximate the optimal solution?

In the statistics literature, density estimation has been a central topic in nonparametric statistics \citep{nemirovski2000topics, tsybakov2009introduction, wassermann2006all}. The minimax optimal rate of convergence has been understood fairly well, for a wide range of density function classes quantified by its smoothness. We would like to point out a simple yet convincing connection between two fields: in nonparametric statistics, the model grows in size to accommodate the complexity of the data, which is reminiscent of the model complexity in GANs, and more generally in the deep neural networks.

The current paper studies the GAN framework for learning densities from a nonparametric point of view. The focus of the current paper is \textit{not} on the optimization side of how to solve for $\tilde{\mu}_n$ efficiently, rather on the statistical front: how well $\tilde{\mu}_n$ estimates a wide range of nonparametric distributions $\nu$ under a wide collection of objective metrics, and how to improve the GAN procedure with better theoretical guarantee such as rate of convergence.

Note the GAN framework mentioned above is flexible. Define the following metric induced by the function class $\mathcal{F}$,
\begin{align*}
	d_{\F}(\mu, \nu) := \max_{f \in \F} \E_{Y \sim \mu} f(Y) - \E_{X \sim \nu} f(X).
\end{align*}
If $\mathcal{F}$ contains all Lipschitz-$1$ functions, then $d_{\F}(\mu, \nu)$ is the Wasserstein-1 metric (Wasserstein-GAN \citep{arjovsky2017wasserstein}). When $\mathcal{F}$ represents all functions bounded by $1$, then $d_{\F}(\mu, \nu)$ is the total variation metric (Radon metric). Let $\mathcal{H}$ be a reproducing kernel Hilbert space (RKHS) and $k(\cdot, \cdot)$ be its kernel. If $\mathcal{F}$ consists of functions in the closure of the span of the set $\{k(\cdot ,x), x \in \Omega\}$, then $d_{\F}(\mu, \nu) = \| d(\mu - \nu)/dx \|_{\mathcal{H}}$ (MMD-GAN \citep{dziugaite2015training,li2015generative}).

\subsection{Contributions}
We summarize the main contributions here. Before introducing the results, let's mention the Sobolev spaces formally defined in Def.~\ref{def:sobolev} and \ref{def:sobolev-frac}. We use $W^{\alpha, p}(L)$ to denote the Sobolev space with smoothness $\alpha \geq 0$, with $\ell_p$-Sobolev norm bounded by $L$, for $p=2, \infty$. Denote the Wasserstein-$1$ metric as $W(\mu, \nu)$.

Consider the density function that lies in Sobolev space with smoothness parameter $\alpha$
	\begin{align*}
		\F_{R} := \left\{ \nu: d \nu(x)/d x \in W^{\alpha,2}(L) \right\}
	\end{align*}
	and the discriminator function class $\F_D = W^{\beta,2}(L)$ with smoothness $\beta$. As $\alpha, \beta$ varies, $\F_R$ describes a wide range of nonparametric densities, and $\F_D$ provides a rich hierarchy of critic metrics. Let $n$ be the sample size, $d$ be the dimension.

\noindent \textbf{Smoothness motivates new procedures.} \quad
	We introduce a new GAN estimator $\tilde{\mu}_n$ based on a regularized/smoothed version of the empirical measure $\tilde{\nu}_n$, which enjoys the rate
	\begin{align*}
		 \E d_{\F_D}(\tilde{\mu}_n, \nu) - \overbrace{ \min_{\mu \in \mu_G} d_{\F_D}(\mu, \nu) }^\text{approximation error} \precsim  n^{ - \frac{\alpha+\beta}{2(\alpha+\beta)+d}}.
	\end{align*}
	In contrast, as long as $\alpha > \frac{\beta}{d/(2\beta) - 1}$, the GAN estimator $\hat{\mu}_n$ with empirical measure $\hat{\nu}_n$ only achieves a considerably slower rate
	\begin{align*}
		 \E d_{\F_D}(\hat{\mu}_n, \nu) -  \min_{\mu \in \mu_G} d_{\F_D}(\mu, \nu) \precsim  n^{-\frac{\beta}{d}} + n^{-\frac{1}{2}},
	\end{align*}
	which doesn't adapt to the smoothness of the target measure $\nu$. Remark that the regularized/smoothed empirical measure estimate $\tilde{\nu}_n$ can theoretically be viewed as a remedy to the known mode-collapse problem \citep{arora2017theoretical} reported in GAN. Our improved theoretical rates complement the reported findings in practice, that introducing ``instance noise'' to smooth the empirical measure fulfills the goal of stabilizing GANs \citep{arjovsky2017towards, sonderby2016amortised}. 

\noindent \textbf{Nonparametric estimation with GAN framework.} \quad
	The GAN framework for nonparametric density estimation enjoys the upper bound
	\begin{align*}
		\sup_{\nu \in \F_R} \left\{ \E d_{\F_D}(\tilde{\mu}_n, \nu) -  \min_{\mu \in \mu_G} d_{\F_D}(\mu, \nu) \right\} \precsim \min\left\{ n^{ - \frac{\alpha+\beta}{2(\alpha+\beta)+d}}, n^{-\frac{\beta}{d}} + n^{-\frac{1}{2}}  \right\}.
	\end{align*}
	One may wonder whether this rate can be significantly improved by other approaches. We show that for any procedure $\nu_n$ based on $n$ samples, the minimax lower bound under the metric $d_{\F_D}$ reads
	\begin{align*}
		\inf_{\nu_n} \sup_{\nu \in \F_R} d_{\F_D}( \nu_n, \nu) \succsim  n^{-\frac{\alpha+\beta}{2\alpha + d} } + n^{-\frac{1}{2}}.
	\end{align*}
	In the case when $d$ is large, the exponent in the upper bound and the minimax lower bounds are very close to each other, in the sense that
	$\frac{\alpha+\beta}{2\alpha + d} = (1 + \mathcal{O}(1/d))\frac{\alpha+\beta}{2(\alpha+\beta)+d} $. 

\medskip

In the case when both the \textit{generator} and \textit{discriminator} are realized by \textit{neural networks}. We establish the following results using the insights gained from above.

\noindent \textbf{Networks can learn densities.} \quad
We make progress on answering how well generative adversarial networks learn nonparametric densities $d\nu(x)/dx \in W^{\alpha, \infty}(L)$ with smoothness parameter $\alpha$, under the Wasserstein metric.
	
	Consider the GAN estimator $\tilde{\mu}_n$ in \eqref{eq:gan}.
	Assume that the discriminator neural network can approximate functions in $W^{1, \infty}(1)$ within error $\epsilon$, and the generator neural network can approximate densities in $W^{\alpha, \infty}(L)$ within $\epsilon$.
	Then $\tilde{\mu}_n$ learns the true distribution $\nu$ with accuracy
	\begin{align*}
		\sup_{d\nu/dx \in W^{\alpha, \infty}(L)} \E W(\tilde{\mu}_n, \nu) \leq C \cdot \left( \epsilon +  n^{ - \frac{\alpha+1}{2\alpha+2+d}} \right).
	\end{align*}
	In addition, for all estimators $\nu_n$ based on $n$ samples, the minimax rate cannot be smaller than
	\begin{align*}
	\inf_{\nu_n} \sup_{d\nu/dx \in W^{\alpha, \infty}(L)}\E W(\nu_n, \nu) \geq c \cdot n^{-\frac{\alpha+1}{2\alpha + d}}.
	\end{align*}
	Here $C, c >0$ are constants independent of $n$. Note that by \cite{yarotsky2016error}, one can approximate $W^{k, \infty}(1)$ within $\epsilon$ accuracy by ReLU networks with $\log(1/\epsilon)$ depth and $\text{poly}(1/\epsilon)$ units for integer $k$. The formal and more general statment can be found in Thm.~\ref{thm:deep-relu-learn}.

\noindent \textbf{Generalization error for deeper ReLU discriminator networks.} \quad
	Let $\F_D = \F^{(\ell)}(V)$ be the collection of functions realized by the feedforward ReLU network with depth $\ell$, where for each unit in the network, the vector $w$ of weights associated with that unit has $\| w \|_1 \leq V$. Let the true density $d\nu/dx \in  W^{\alpha, 2} (L)$. Then there exist a regularized/smoothed GAN estimator $\tilde{\mu}_n$ that satisfies the upper bound,
	\begin{align*}
		\E d_{\F_D}(\tilde{\mu}_n, \nu) -  \min_{\mu \in \mu_G} d_{\F_D}(\mu, \nu) \leq C \cdot L(d+1)^{\frac{1}{2}} V^{\ell} n^{-\frac{\alpha+1}{2\alpha + 2 + d}}.
	\end{align*}
	This improves significantly upon the known bound on $\hat{\mu}_n$ obtained using empirical density --- which is $C \cdot (\log d)^{\frac{1}{2\ell}} (2V)^{(\ell+1)/2} n^{-\frac{1}{2\ell}}$ \citep{anthony2009neural} --- as long as $\alpha > \frac{d}{2(\ell-1)} - 1$, which allows for better guarantee for deeper discriminator networks.

\subsection{Preliminaries}
\label{sec:prelim}

Let's introduce the notations, and function spaces used in this paper. During the discussion, we restrict the feature space to be $\Omega = [0,1]^d$. We use $\mu, \nu$ to denote the measure, and the Radon-Nikodym derivatives $d\mu(x)/dx$, $d\nu(x)/dx$ to denote the corresponding density functions. $\| f\|_p := \left(\int_{\Omega} |f(x)|^p dx \right)^{1/p}$ denotes the $\ell_p$-norm under the Lebesgue measure, for $1\leq p \leq \infty$. With slight abuse of notation, for probability measures $\mu, \nu$ we denote $ \| \mu \|_p := \| d\mu/dx \|_{p}$ and similarly for $\| \mu - \nu\|_p =\| d\mu/dx - d\nu/dx \|_{p} $. For vector $w$, $\|w \|_p$ denotes the vector $\ell_p$-norm.

\begin{defn}[Sobolev space: $k \in \mathbb{N}$]
	\label{def:sobolev}
	For integer $k$, define the Sobolev space $W^{k, p}(L)$ ($p=2,\infty$)
	\begin{align*}
		W^{k, 2}(L) &:= \left\{ f \in \Omega \rightarrow \mathbb{R} : ( \sum_{|\alpha| \leq k} \| D^{(\alpha)} f \|_{2}^2 )^{1/2} \leq L \right\}, \\
		W^{k, \infty}(L) &:= \left\{ f \in \Omega \rightarrow \mathbb{R} : \max_{|\alpha| \leq k} \| D^{(\alpha)} f \|_{\infty} \leq L \right\},
	\end{align*}
	where $\alpha$ is a multi-index and $D^{(\alpha)}$ denotes the $\alpha$-weak derivative.
\end{defn}
The definition naturally extends to fractional $\alpha \in \mathbb{R}$ through the Bessel potential, with $\mathcal{F} f(\xi): \mathbb{R} \rightarrow \mathbb{C}$ denotes the Fourier transform of $f$, and $\mathcal{F}^{-1}$ as its inverse.
\begin{defn}[Sobolev space: $\alpha \in \mathbb{R}$]
	\label{def:sobolev-frac}
	For $\alpha \in \mathbb{R}$, the Sobolev space $W^{\alpha, 2}(L)$ definition extends to non-integer $\alpha$,
	\begin{align*}
		W^{\alpha, 2}(L) := \left\{ f \in \Omega \rightarrow \mathbb{R}: \left\| \mathcal{F}^{-1}\left[ (1+|\xi|^2)^{\frac{\alpha}{2}} \mathcal{F} f(\xi) \right] \right\|_2 \leq L \right\}.
	\end{align*}
\end{defn}
Another equivalent definition of Sobolev space for $p=2$ is through the coefficients of the generalized Fourier series, which is also referred to as the Sobolev ellipsoid.
\begin{defn}[Sobolev ellipsoid]
	Let $\theta = \{\theta_\xi, \xi = (\xi_1, \ldots, \xi_d) \in \mathbb{N}^d \}$ collects the coefficients of the generalized Fourier series, define
	\begin{align*}
		\Theta^{\alpha}(L) := \left\{ \theta \in \mathbb{N}^d \rightarrow \mathbb{R} : \sum_{\xi \in \mathbb{N}^d} (1+\sum_{i=1}^d \xi_i^2 )^{\alpha} \theta^2_{\xi} \leq L^2 \right\}.
	\end{align*}
\end{defn}
It is clear that $\Theta^\alpha(L)$ (frequency domain) is an equivalent representation of $W^{\alpha, 2}(L)$ (spatial domain) in $L^2(\mathbb{N}^d)$ for trigonometric Fourier series. For more details on Sobolev spaces, we refer to \cite{nemirovski2000topics, tsybakov2009introduction, nickl2007bracketing}.

Denote $A(n) \precsim n^{\alpha}$, if $\varlimsup\limits_{n\rightarrow \infty} \frac{\log A(n)}{\log n} \leq \alpha$, holding other parameters fixed, similarly $A(n) \succsim n^{\alpha}$ if $\varliminf\limits_{n\rightarrow \infty} \frac{\log A(n)}{\log n} \geq \alpha$. Denote $A(n) \asymp n^{\alpha}$ if $A(n) \precsim n^{\alpha}$ and $A(n) \succsim n^{\alpha}$. $[K]:=\{ 0, 1, \ldots, K\}$ denotes the index set, for any $K \in \mathbb{N}$.

\section{A nonparametric view of the GAN framework}

\subsection{An oracle inequality}

The following oracle inequality holds for GAN.

\begin{thm}[Oracle inequality for GAN]
	\label{thm:oracle}
	Let $\F$ be any critic function class. Denote $\mu_n$ as the solution (w.r.t. the empirical estimate $\nu_n$) to GAN with generator $\mu_G$ and discriminator $\F_D$
	\begin{align*}
		\mu_n = \argmin_{\mu \sim \mu_{G}} \max_{f \in \F_{D}} \E_{Y \sim \mu} f(Y) - \E_{X \sim \nu_n} f(X).
	\end{align*}
	Then the following decompositions hold for any distribution $\nu$,
	\begin{align}
		\small
		&d_{\F_D}(\mu_n, \nu) \leq  \overbrace{\min_{\mu \in \mu_G} d_{\F_D}(\mu, \nu)}^\text{approx. err.} + \overbrace{d_{\F_D}(\nu, \nu_n) + d_{\F_D}(\nu_n, \nu)}^\text{stat. err.}, \label{eq:oracle-1}\\
		&d_{\F}(\mu_n, \nu) \leq \overbrace{\min_{\mu \in \mu_G} d_{\F_D}(\mu, \nu)}^\text{generator approx. err.} +~ (1+ \| \nu_n \|_1) \cdot \overbrace{\max_{f \in \F} \min_{f' \in \F_D} \| f - f' \|_{\infty}}^\text{discriminator approx. err.} +  \overbrace{d_{\F_D}(\nu, \nu_n) + d_{\F}(\nu_n, \nu)}^\text{stat. err.}. \label{eq:oracle-2}
	\end{align}
\end{thm}

Let's remark on the decompositions. Eqn.~\eqref{eq:oracle-1} use $d_{\F_D}$ as the objective evaluation metric. The first term is the best approximation error within the generator class, even if we have population access to the true measure $\nu$. The second term is the statistical error, also called the generalization error, due to the fact that we have only finite $n$-samples. Eqn.~\eqref{eq:oracle-2} employs a different $d_{\F}$ as the objective metric, while using the mis-matched discriminator class $\F_D$ in the GAN procedure. The first term describes the approximation error induced by the generator, the second term corresponds to how well the discriminator serves as a surrogate for the objective metric, and the third term is the statistical error.

Using the above theorem, choose $\F_D$ as the critic metric, we need to study the excess loss
\begin{align*}
	\E d_{\F_D}(\mu_n, \nu) -  \min_{\mu \in \mu_G} d_{\F_D}(\mu, \nu) \leq 2 \E d_{\F_D}(\nu, \nu_n).
\end{align*}
We will start with a crude bound when $d\nu_n(x)/dx = \frac{1}{n}\sum_{i=1}^n \delta_{X_i}(x)$ is chosen as the empirical density. It turns out that one can significantly improve this bound through choosing a better ``regularized'' or ``smoothed'' version of empirical measure, in the context of learning nonparametric densities. The empirical measure is not optimal because one can leverage the complexity/smoothness of the measure $\nu$, and the complexity of $\mathcal{F}_D$ to improve upon the generalization error.

\subsection{Upper bound for arbitrary density}

Let's start with a simple bound on $\E d_{\F_D}(\nu, \nu_n)$ with $\nu_n = \hat{\nu}_n$, the empirical measure. We will illustrate why this is suboptimal to feed to GAN to learn smooth densities, and further how to improve by employing a regularized estimate $\tilde{\nu}_n$ as a plug-in for GAN.
\begin{lem}[Arbitrary density]
	\label{lem:sym}
	Take $d \hat{\nu}_n(x)/dx = \frac{1}{n}\sum_{i=1}^n \delta_{X_i} (x)$, then
	\begin{align}
		\E d_{\F_D}(\nu, \hat{\nu}_n) \leq 2 \E_{\nu^{\otimes n}} \E_{\epsilon} \sup_{f\in \F_D} \frac{1}{n} \sum_{i=1}^n \epsilon_i f(X_i).
	\end{align}
	Assuming $\sup_{f \in \F_D} \| f \|_{\infty} \leq 1$, one has the standard entropy integral bound,
	\begin{align*}
		\E d_{\F_D}(\nu, \hat{\nu}_n) \leq  2 \E_{\nu^{\otimes n}} \inf_{0<\delta<1/2} \left( 4\delta + \frac{8 \sqrt{2}}{\sqrt{n}} \int_{\delta}^{1/2}  \sqrt{\log \mathcal{N}(\epsilon, \F_{D}, \| \cdot \|_{n}) } d \epsilon \right),
	\end{align*}
	where $\| f \|_n:= \sqrt{1/n \sum_{i=1}^n f(X_i)^2}$ is the empirical $\ell_2$-metric on data $\{X_i\}_{i=1}^n$.
\end{lem}

The above lemma is a standard symmetrization followed by the Dudley entropy integral.
When applying the symmetrization lemma, one often discard the distribution information about $\nu$, and thus end up with distribution-independent guarantees. The reason is that one typically bound the empirical $\ell_2$-covering number by $\ell_\infty$-covering number on the function class (independent of $X_i$)
	\begin{align*}
		&\| f- f' \|_n = \left[ \frac{1}{n} \sum_{i=1}^n \left( f(X_i) - f'(X_i) \right)^2 \right]^{1/2} \leq \sup_{x} \| f(x) - f'(x) \| = \| f - f' \|_{\infty}, \\
		& \text{which implies}~~\log \mathcal{N}(\epsilon, \F_{D}, \| \cdot \|_{n}) \leq \log \mathcal{N}(\epsilon, \F_{D}, \| \cdot \|_{\infty}).
	\end{align*}
Plugging in the entropy estimate $\log \mathcal{N}(\epsilon, \F_{D}, \| \cdot \|_{\infty})$ for various functional classes (\cite{nickl2007bracketing} and reference therein), one can easily derive the following corollary.

\begin{coro}[Rates for arbitrary density]
	\label{coro:arb}
	If $\F_D$ denotes $L$-Lipschitz function (Wasserstein-GAN) on $\mathbb{R}^d$, $d\geq 2$, one has
	\begin{align*}
		\E d_{\F_D}(\hat{\mu}_n, \nu) -  \min_{\mu \in \mu_G} d_{\F_D}(\mu, \nu) \leq 2\E d_{\F_D}(\nu, \hat{\nu}_n) \precsim n^{-\frac{1}{d}}.
	\end{align*}
	If $\F_D$ denotes Sobolev space $W^{\beta, 2}(L)$ on $\mathbb{R}^d$, one has
	\begin{align*}
		\E d_{\F_D}(\hat{\mu}_n, \nu) -  \min_{\mu \in \mu_G} d_{\F_D}(\mu, \nu) \leq 2\E d_{\F_D}(\nu, \hat{\nu}_n) \precsim  n^{-\frac{\beta}{d}} + \frac{\log n}{\sqrt{n}}.
	\end{align*}

\end{coro}

It is easy to see that overlooking the distributional information --- by simply going through symmetrization and empirical processes theory --- can lead to suboptimal result when the true density is smooth. Roughly speaking, the symmetrization approach treats the true density as a highly non-smooth one with $\alpha = 0$ (say the empirical density). Next, we will investigate how to improve GAN when the true density lies in Sobolev spaces $W^{\alpha, 2}(L), \alpha>0$.

\subsection{Smoothness helps: improved upper bound for Sobolev spaces}

Now we show that one can achieve faster rates in the GAN framework for density estimation, simultaneously leveraging the smoothness information in the density $d\nu$ and the metric $\F_D$.

Suppose the density function $d\nu(x)/d x \in W^{\alpha,2}(L)$, and $\F_D = W^{\beta,2}(L)$. Claim that there exist $\tilde{v}_n$, a smoothed/regularized version of empirical measure such that plugging into GAN will result in a faster rate that adapts to the smoothness of the true density. Remark that in practice, one can take $\tilde{v}_n$ as a kernel density estimate $d\tilde{v}_n(x)/dx = \frac{1}{nh}\sum_{i=1}^n K((X_i-x)/h)$, where sampling a mini-batch of data from $\tilde{\nu}_n$ is just as simple for stochastic gradient updates in GAN optimization.

\begin{thm}[Nonparametric estimation with GAN framework]
	\label{thm:non-param}
	Consider the density function class that lies in Sobolev space with smoothness parameter $\alpha$, and the discriminator function class with smoothness $\beta$,
	\begin{align*}
		\F_{R} := \left\{ \nu: \frac{d \nu(x)}{d x} \in W^{\alpha,2}(L) \right\},~~ \F_D = W^{\beta,2}(L).
	\end{align*}
	Then there exist a smoothed version of empirical measure $\tilde{\nu}_n$ that satisfies
	\begin{align*}
		 \sup_{\nu \in \F_{R}} \E d_{\F_D}(\nu, \tilde{\nu}_n) \precsim n^{-\frac{\alpha+\beta}{2(\alpha+\beta)+d}}.
	\end{align*}
\end{thm}

\begin{remark}
	\rm
	\label{rmk:nonparam-rate}
	Theorem~\ref{thm:oracle} tells us the rate of convergence satisfies the following upper bound
	\begin{align*}
	 \sup_{\nu \in \F_R} \left\{ \E d_{\F_D}(\tilde{\mu}_n, \nu) -  \min_{\mu \in \mu_G} d_{\F_D}(\mu, \nu) \right\} &\leq 2 \sup_{d\nu/d x \in W^{\alpha,2}(L)} \E d_{\F_D}(\nu, \tilde{\nu}_n) \precsim n^{-\frac{\alpha+\beta}{2(\alpha+\beta)+d}}.
	\end{align*}
	Compare with the bound in Corollary~\ref{coro:arb},
	$
	\E d_{\F_D}(\nu, \hat{\nu}_n) \precsim n^{-\frac{\beta}{d}} + \frac{\log n}{\sqrt{n}},$ one can see that
	when $\beta < d/2$, if the true density is mildly smooth as long as
	$$
	\scriptsize
	\alpha > \frac{\beta^2}{d/2 - \beta},
	$$
	this nonparametric upper bound achieves a faster rate than the symmetrization bound which ignores the smoothness of the distribution. Remark in the real applications of GAN, the dimension $d$ is rather large (in images usually $d>10^3$), and $\beta$ is fairly small for a strong metric ($\beta=1$ in Wasserstein GAN), so this condition is satisfied for any $\alpha > 0.002$.
	In summary, we write out the explicit rate of convergence,
	\begin{equation*}
		\scriptsize
			\inf_{\tilde{\mu}_n} \sup_{\nu \in \F_R} \left\{ \E d_{\F_D}(\tilde{\mu}_n, \nu) -  \min_{\mu \in \mu_G} d_{\F_D}(\mu, \nu) \right\} \precsim
			\begin{cases}
				n^{-\frac{1}{2}}, & \text{if}~\beta \geq \frac{d}{2} \\
				n^{-\frac{\beta}{d}}, & \text{if}~\beta < \frac{d}{2}~\text{and}~ \alpha < \frac{\beta}{d/2\beta-1} \\
				n^{-\frac{\alpha+\beta}{2(\alpha+\beta)+d}}, & \text{if}~\beta < \frac{d}{2}~\text{and}~ \alpha \geq  \frac{\beta}{d/2\beta-1}
			\end{cases}.
	\end{equation*}
	One final remark is that the level of smoothing in $\tilde{\nu}_n$ is more aggressive than that in classic nonparametric regression literature (which results in a slower rate $n^{-\frac{\alpha}{2\alpha + d}}$), as we can utilize at the same time the smoothness information in the discriminator metric $\F_D$.
 \end{remark}

Let's broaden the discussions slightly by considering different base measures $\pi(x)$ beyond the Lebesgue measure, and the \textit{generalized Fourier basis}. One can think of $\pi(x) $ as the uniform measure on $[0,1]^d$ or the product Gaussian measure.
An equivalent formulation of the GAN problem that translates learning a distribution $d \mu(x) = g(x) d\pi(x)$ to learning the importance score/density ratio function $g(x)$ with respect to a base measure $\pi(x), x \in \mathcal{X}$, can be viewed as
\begin{align*}
	\scriptsize
	& \min_{\mu \sim \mu_{G}} \max_{f \in \F_{D}} \E_{X \sim \mu} f(X) - \E_{Y \sim \nu} f(Y) = \min_{\mu \sim \mu_{G}} \max_{f \in \F_{D}} \int f(x) \frac{d \mu(x)}{d \pi(x)} d\pi(x) - \int f(y) \frac{d \nu(y)}{d\pi(y)} d\pi(y) \\
	&= \min_{g \in \F_{G}} \max_{f \in \F_{D}} \langle f, g \rangle_{\pi} - \langle f, h \rangle_{\pi}, \quad \quad \text{where}~ g = \frac{d \mu}{d \pi}, h = \frac{d \nu}{d \pi},
 \end{align*}
$\langle f, g \rangle_{\pi} := \int f(x) g(x) d \pi(x)$, and $\F_G := \{g: g = d\mu/d\pi, \mu \in \mu_G \}$.
Consider the \textit{generalized Fourier basis} $\{\psi_\xi, \xi \in \mathbb{N}^d \}$ w.r.t the base measure $\pi(x)$ that satisfies
$
	\langle \psi_\xi, \psi_\xi \rangle_{\pi} = 1, ~\text{and}~ \langle \psi_\xi, \psi_{\xi'} \rangle_{\pi} = 1~\text{for $\xi \neq \xi'$},
$
where $\xi$ is the multi-index.
For any function $g(x)$, one can represent the function in the generalized Fourier basis
$
	g(x)  = \sum_{\xi \in \mathbb{N}^d} \theta_{\xi}(g) \psi_{\xi}(x),~\text{where $\theta(g) \in \mathbb{N}^d$ denotes the coefficients of $g$}.
$
The Sobolev ellipsoid --- the coefficients $\theta(g)$ lies in --- quantifies the smoothness of the function $g$. As a special case, when the base measure is uniform distribution on $[0,1]^d$,
$
	\theta(g) \in \Theta^{\alpha}(L) \Leftrightarrow g \in W^{\alpha, 2}(L).
$
In general, for any base measure $\pi(x)$ and its corresponding Fourier series $\{\psi_\xi, \xi \in \mathbb{N}^d \}$, one can easily extend Theorem~\ref{thm:non-param} using Sobolev ellipsoid $\Theta^{\alpha}(L)$.
\begin{coro}
	\label{coro:non-param}
	Assume the generalized Fourier basis are bounded $\max_{\xi} \|\psi_\xi(x)\|_{\infty} \leq C$.
	Consider the density function class with smoothness $\alpha$, and the discriminator function class $\F_D$ with smoothness $\beta$
	\begin{align*}
		\F_{R} &:= \left\{ \nu: d \nu/d\pi = g, ~\text{and}~\theta(g) \in \Theta^{\alpha,2}(L) \right\},\\
		\F_{D} &:= \left\{ f: \theta(f) \in \Theta^{\beta,2}(L) \right\}.
	\end{align*}
	Then there exist a smoothed version of empirical measure $\tilde{\nu}_n$ that satisfies
	\begin{align*}
		 \sup_{\nu \in \F_{R}} \E d_{\F_D}(\nu, \tilde{\nu}_n) \precsim  n^{-\frac{\alpha+\beta}{2(\alpha+\beta)+d}}.
	\end{align*}
\end{coro}

\subsection{Minimax lower bound}

Can the rate obtained by GAN framework be significantly improved by any other procedure? We consider in this section the minimax lower bound for the intrinsic difficulty of nonparametric estimation under the GAN discriminator metric. Here we consider a nonparametric function estimation problem on the Sobolev ellipsoid $\Theta^{\alpha}(L)$, which is statistically equivalent to the density estimation problem over the Sobolev space $W^{\alpha, 2}(L)$ asymptotically \citep{tsybakov2009introduction}. A separate lower bound for the density estimation within the smaller \textit{H\"{o}lder class} $W^{\alpha, \infty}(L)$ is proved in Thm.~\ref{thm:deep-relu-learn}.

Consider the problem of estimating the function $g(x)  = \sum_{\xi \in \mathbb{N}^d} \theta_{\xi}(g) \psi_{\xi}(x)$, where the coefficients belongs to a Sobolev ellipsoid $\theta(g) \in \Theta^\alpha(L)$. What one observes are i.i.d normal sequences $\{Y_{\xi}, \xi \in \mathbb{N}^d \}$
\begin{align}
	\label{eq:gaussian-sequence}
	Y_\xi = \theta_{\xi}(g) + \frac{1}{\sqrt{n}} Z_\xi, ~Z_\xi \sim N(0,1).
\end{align}
Based on observations $\{Y_{\xi}, \xi \in \mathbb{N}^d \}$,
we want to know how well one can estimate $g$ w.r.t the following metric
\begin{align*}
	 d_{\F_D}( \hat{g}, g) = \sup_{f \in \F_D} \langle f, \hat{g} \rangle_\pi - \langle f, g \rangle_\pi, ~\text{where}~ \F_{D} := \left\{ f: \theta(f) \in \Theta^{\beta}(L) \right\}.
\end{align*}
\begin{thm}[Minimax lower bound]
	\label{thm:lower-bound}
	Consider the function class with smoothness $\alpha$ and the discriminator function class $\F_D$ with smoothness $\beta$
	\begin{align*}
		\F_{R} := \left\{ \nu: d \nu/d\pi = g, ~\text{and}~\theta(g) \in \Theta^{\alpha}(L) \right\}, \quad
		\F_{D} := \left\{ f: \theta(f) \in \Theta^{\beta}(L) \right\}.
	\end{align*}
	In the Gaussian sequence model \eqref{eq:gaussian-sequence}, for any estimator $\nu_n$,
	\begin{align*}
		\inf_{\nu_n} \sup_{\nu \in \F_R} d_{\F_D}( \nu_n, \nu) \succsim n^{-\frac{\alpha+\beta}{2\alpha + d}}.
	\end{align*}
\end{thm}
\begin{remark}
\rm
Let us compare our lower bound with the upper bound in Remark~\ref{rmk:nonparam-rate}. When the generator function class is rich enough so that $\mu_G$ contains the $\nu$, the upper bound becomes
\begin{align*}
	\inf_{\tilde{\mu}_n} \sup_{\nu \in \F_R} \E d_{\F_D}(\tilde{\mu}_n, \nu)& = \inf_{\tilde{\mu}_n} \sup_{\nu \in \F_R} \left\{ \E d_{\F_D}(\tilde{\mu}_n, \nu) -  \min_{\mu \in \mu_G} d_{\F_D}(\mu, \nu) \right\} \\
	& \leq 2 \inf_{\tilde{\nu}_n} \sup_{\nu \in \F_R} \E d_{\F_D}(\tilde{\nu}_n, \nu) \precsim \min\left\{ n^{ - \frac{\alpha+\beta}{2(\alpha+\beta)+d}}, n^{-\frac{\beta}{d}} + n^{-\frac{1}{2}}  \right\}.
\end{align*}
In contrast the lower bound reads
\begin{align*}
	\inf_{\nu_n} \sup_{\nu \in \F_R} d_{\F_D}( \nu_n, \nu) \succsim  n^{-\frac{\alpha+\beta}{2\alpha + d} } + n^{-\frac{1}{2}}.
\end{align*}
Remark again, when take as a special case the base measure $\pi(x)$ being uniform on $[0,1]^d$, $\F_R = W^{\alpha,2}(L)$ and $\F_D = W^{\beta,2}(L)$ are the usual Sobolev spaces. In the case when $d$ is large, the exponents in the upper and lower bounds are very close to each other, in the sense that
	$\frac{\alpha+\beta}{2\alpha + d} = (1 + o_d(1))\frac{\alpha+\beta}{2(\alpha+\beta)+d} $.
 \end{remark}

\section{How well networks learn densities}

In this section, we answer in a \textit{quantitative} way that when both the \textit{generator} class and the \textit{discriminator} class represented by deep ReLU networks are rich enough, one do learn the distribution. One can view our results as establishing rate of convergence and fundamental difficulty for learning the distribution under the GAN framework, for a wide range of densities with different smoothness properties. It builds a more detailed theory upon the seminal work of \cite{goodfellow2014generative}, where they proved that when sample sizes, generator sizes, and discriminator sizes are all infinite, one learns the distribution.

\subsection{Deep ReLU networks can learn densities}

Let's remark on the universal approximation property of deep networks first before introducing the result. It is well known that deep neural networks are universal approximators \citep{hornik1989multilayer}. In particular, \cite{yarotsky2016error} constructed a fixed ReLU network architecture, denoted as $\F_{relu}^{\alpha}(\epsilon)$, that enjoys the following two properties:
\begin{itemize}
	\item It approximates all functions in $W^{\beta, \infty}(1)$ in the following sense $$\max_{f \in W^{\beta, \infty}(1)} \min_{f' \in \F_{relu}^{\beta, \infty}(\epsilon)} \| f - f' \|_\infty \leq \epsilon.$$
	\item It has depth at most $\mathcal{O}(\log 1/\epsilon)$ and at most $\mathcal{O}(\epsilon^{-d/\beta} \log 1/\epsilon)$ weights and computation units.
\end{itemize}

With the above in mind, we are ready to state the following theorem.
\begin{thm}
	\label{thm:deep-relu-learn}
	Let $\alpha, \beta \in \mathbb{N}$, $n$ be the sample size, and $\epsilon > 0$ be the approximation accuracy. Let $L, M>0$ be some fixed constants.
	Assume that:
	\begin{enumerate}
		\item[A.1] The \textit{discriminator} network class can approximate the Sobolev space $W^{\beta, \infty}(1)$, in the sense that
		$\max_{f \in W^{\beta, \infty}(1)} \min_{f' \in \F_D} \| f - f' \|_\infty \leq \epsilon$;
		\item[A.2] The \textit{generator} network class can approximate the Sobolev space $\nu \in W^{\alpha, \infty}(L)$, in the sense that for any true density $d\nu/dx \in W^{\alpha, \infty}(L)$, $\min_{\mu \in \mu_G} \| \mu - \nu \|_\infty \leq \epsilon$;
		\item[A.3] The \textit{discriminator} network class is not overly complex in the sense that $\F_D \subset W^{\beta, 2}(L) \cap W^{0, \infty}(M)$. In other words, the weak derivatives are integrable and the function is bounded.
	\end{enumerate}
	Then, there exist a GAN estimator $\tilde{\mu}_n$ that satisfies the upper bound
	\begin{align}
		\sup_{\nu \in W^{\alpha, \infty}(1)} \E d_{W^{\beta, \infty}(1)}(\tilde{\mu}_n, \nu) \leq C_1 \epsilon + C_2 \min\left\{ n^{ - \frac{\alpha+\beta}{2(\alpha+\beta)+d}}, n^{-\frac{\beta}{d}} + n^{-\frac{1}{2}}  \right\}.
	\end{align}
	In addition, the minimax lower bound for any estimator $\nu_n$ based on $n$ samples satisfies,
	\begin{align}
		\inf_{\nu_n} \sup_{d\nu/dx \in W^{\alpha, \infty}(L)}\E d_{W^{\beta, \infty}(1)}(\nu_n, \nu) \geq C_3 \left( n^{-\frac{\alpha+\beta}{2\alpha + d}} + n^{-\frac{1}{2}}\right).
	\end{align}
	Here $C_i, i = 1,2,3$ are constants independent of $n$.
\end{thm}

\begin{remark}
	\rm
	Note the lower bound here is harder than that in Theorem~\ref{thm:lower-bound} because: (1) the class is in fact H\"{o}lder (subset of the Sobolev), $W^{\alpha, \infty}(L) \subset W^{\alpha, 2}(L)$, (2) for density estimation directly. 
	
	Let us make few remarks on the objective metrics and rates here.
	When $\beta = 1$, the objective metric is equivalent to the Wasserstein-$1$ metric (Lemma~\ref{lem:rela.order}). Therefore with properly chosen discriminator and generator network GAN enjoys the upper bound $n^{-\frac{\alpha + 1}{2\alpha + 2 + d}}$, while the minimax lower bound being $n^{-\frac{\alpha + 1}{2\alpha + d}}$. Consider the following two scenarios:
	\begin{enumerate}
		\item When the dimension $d$ is very large, for a fixed smoothness parameter $\alpha$, these two rates are very close. In this case, to obtain an $\epsilon$-error in Wasserstein metric, GAN requires exponential number of samples (in dimension) $n = \epsilon^{-(2 + \frac{d}{\alpha+1})}$, so does any other procedure as indicated by the lower bound.
		\item When the the smoothness parameter scales with the dimension, say $\alpha+1 \geq  d/C_0$ for some constant $C_0>0$, then GAN only requires a polynomial number of samples $n = \epsilon^{-(2+C_0)}$. In the large dimension setting when $d \rightarrow \infty$, any other procedure will require at least $n = \epsilon^{-(2+C_0)}$ as well.
	\end{enumerate}
	When $\beta = 0$, the objective metric is the total variation distance. In this case, the rate achieved by GAN is rate optimal for all $\alpha$, with $n^{-\frac{\alpha}{2\alpha+d}}$. It is also easy to see this as $\| \mu -\nu \|_1 \leq \| \mu -\nu \|_2$, and the nonparametric estimation rate for $\ell_2$-error is known to be $n^{-\frac{\alpha}{2\alpha+d}}$.
 \end{remark}

\begin{remark}
	\rm
	Let us comment on the assumptions used in the theorem. We consider the function class realized by ReLU networks. Define $c(\alpha, \beta, d) = \frac{\alpha+\beta}{2(\alpha+\beta)+d} \vee (\frac{\beta}{d} \wedge \frac{1}{2})$ and let's choose $\epsilon \asymp n^{-c(\alpha, \beta, d)}$.

	The assumption A.1 on the discriminator network class can be realized by subsuming a specific fixed architecture \citep{yarotsky2016error} as a subclass $\F_D \supseteq \F_{relu}^{\beta, \infty}(\epsilon)$, where the network $\F_{relu}^{\beta, \infty}(\epsilon)$ has depth $\mathcal{O}(c(\alpha, \beta, d) \log n)$, and $\mathcal{O}^*( n^{\frac{d}{\beta} c(\alpha, \beta, d)})$ weights and computational units. Roughly, this means that one may need a deep discriminator network (scales logarithmically with number of samples in depth, and polynomially in number of units) to recover the density in a strong sense such as under the Wasserstein metric.

	The assumption A.2 on the generative network class can also be realized if we reformulate the GAN in an importance sampling way: given uniform samples $X \in [0,1]^d$, the generative network outputs a importance/density score $\mu(X)$ realized by function realized by the network. Again, the deep network $\mu_G \supseteq \F_{relu}^{\beta, \infty}(\epsilon)$ with depth $\mathcal{O}(c(\alpha, \beta, d) \log n)$ and $\mathcal{O}^*( n^{\frac{d}{\beta} c(\alpha, \beta, d)})$ weights does the job. Note in practice, the generative network learns a transformation of $X$ with $\mu(X)$ as the density, instead of assigning the importance score. We do acknowledge that learning the transformation mapping (represented by the network) from the input random measure to the target measure has computational advantage over the importance sampling approach. However, overlooking the computation, output an importance score is also plausible to learn a distribution theoretically.

	The assumption A.3 can be relaxed. In fact, one can prove a variant of the theorem without assumption 3, in the sense
	\begin{align*}
		\E d_{W^{\beta, \infty}(1)}(\tilde{\mu}_n, \nu)\leq  \E d_{\F_D}(\tilde{\nu}_n, \nu) +  n^{-c(\alpha, \beta, d)}.
	\end{align*}
	It formalizes the intuition that if we choose a too strict metric as the discriminator ($\F_D$ too rich), even though we are evaluating on a weaker metric (induced by a smaller function space $W^{\beta, \infty}(1)$), there is a price to pay in the rates for choosing the overly complex space $\F_D$ as discriminator.

 \end{remark}

\subsection{Error rates for deeper ReLU discriminator networks}
In the GAN formulation, the discriminator class $\F_D$ is represented by functions realized by a neural network with a certain architecture. In this section, we will apply our theory to obtain bounds for a wide range of feed-forward multi-layer networks with ReLU activation as the discriminator metric.

Denote $\sigma(x) = \max(x, 0)$ as the ReLU activation. Consider the following feed-forward multi-layer network:
\begin{itemize}
	\item The network has $\ell$ layers, input $x \in [0,1]^d$,
	\item There is constant $V\geq 1$ such that for each unit in the network, the vector $w$ of weights associated with that unit has $\| w \|_1 \leq V$.
\end{itemize}
Mathematically, one can define the function class induced by the network through the recursive definition:
$\F^{(0)}(V) = \{ x \rightarrow x_i: x = (x_1, \ldots, x_d)\in [0,1]^d, i \in [d] \} \cup \{0, 1\}$, and for any $i \geq 1$,
\begin{align}
	\F^{(i)}(V) = \left\{ \sum_j w_j \sigma(f_j): \forall j,~ f_j \in \bigcup_{k=0}^{i-1} \F^{(k)}(V), \| w \|_1 \leq V \right\}.
\end{align}

\begin{thm}[GAN with ReLU discriminator network]
	\label{thm:relu-discrimin}
	Let $\F_D = \F^{(\ell)}(V)$ be the collection of functions realized by the discriminator ReLU network with depth $\ell$, and the true density $d\nu(x)/dx \in  W^{\alpha, 2} (L)$ lies in a Sobolev space with smoothness parameter $\alpha$. Then there exist a GAN estimator $\tilde{\mu}_n$ that satisfies the upper bound,
	\begin{align}
		\label{eq:non-para-deep-relu}
		\E d_{\F_D}(\tilde{\mu}_n, \nu) -  \min_{\mu \in \mu_G} d_{\F_D}(\mu, \nu) \leq C L(d+1)^{\frac{1}{2}} V^{\ell} n^{-\frac{\alpha+1}{2\alpha + 2 + d}}.
	\end{align}
\end{thm}

\begin{remark}
	\rm
Let's remark on how our bound improves upon what is known, and when one can improve GAN by using a better estimate than the empirical measure $\hat{\nu}_n$.

Recall the covering number bound (Thm. 14.17) for deeper networks in \cite{anthony2009neural}, one knows that
\begin{align*}
	\log \mathcal{N}(\epsilon, \F^{(\ell)}(V), \| \cdot \|_{n}) \leq \frac{1}{2}\left( \frac{1}{\epsilon} \right)^{2\ell} (2V)^{\ell(\ell+1)} \log(2d+2).
\end{align*}
Plug in the covering number bound into our entropy integral in Lemma~\ref{lem:sym}, we know that GAN solution $\hat{\mu}_n$ with the empirical measure $\hat{\nu}_n$ as input enjoys (for $\ell \geq 2$)
	\begin{align}
		\E d_{\F_D}(\hat{\mu}_n, \nu) -  \min_{\mu \in \mu_G} d_{\F_D}(\mu, \nu)  &\leq C \inf_{\delta} \left( \delta + \frac{1}{\sqrt{n}} \int_{\delta}^{\infty} \sqrt{\log \mathcal{N}(\epsilon, \F^{(\ell)}(V), \| \cdot \|_{n})} d\epsilon \right) \nonumber\\
		&\leq C (\log d)^{\frac{1}{2\ell}} (2V)^{(\ell+1)/2} n^{-\frac{1}{2\ell}}.\label{eq:chainning-deep-relu}
	\end{align}
When compare these two bounds in \eqref{eq:non-para-deep-relu} and \eqref{eq:chainning-deep-relu}, we see that the $V^{\ell}$ factor occurs in both bounds, but note that our bound is better when
\begin{align*}
	\alpha > \frac{d}{2(\ell-1)} - 1 \Leftrightarrow n^{-\frac{\alpha+1}{2\alpha+2+d}} < n^{-\frac{1}{2\ell}}.
\end{align*}
More concretely, consider the following two extreme scenarios:
\begin{enumerate}
	\item Density is rather smooth, say $\alpha + 1 \geq d/(2C_0)$ for some fixed $C_0$, then as long as $\ell \geq C_0 + 1$, our bound is better.
	\item Network is deep, say $\ell-1 \geq d/(2C_0)$, then as long as $\alpha \geq C_0 - 1$, our bound is better.
\end{enumerate}
In other words, we do achieve a faster rate to substitute the empirical measure $\hat{\nu}_n$ by a regularized (smoothed) version $\tilde{\nu}_n$ when learning a smooth density (large $\alpha$) using GAN with a deeper discriminator network (large $\ell$). And our bound enjoys significant advantage when the discriminator network is very deep as the exponent in $n$ doesn't depends on the depth. Remark that both these two bounds does not depend on the number of computing units explicitly as long as the $\ell_1$ norm of each unit is bounded, i.e., they are scale sensitive.

 \end{remark}

\section{Technical Proofs}
\begin{proof}[Proof of Theorem~\ref{thm:oracle}] \quad
	For any $\mu \in \mu_G$, we know that due to the optimality of GAN
	\begin{align*}
		d_{\F_D}(\mu, \nu_n) - d_{\F_D}(\mu_n, \nu_n) \geq 0.
	\end{align*}
	Therefore for any $\mu \in \mu_G$,
	\begin{align*}
		d_{\F}(\mu_n, \nu) & \leq d_{\F}(\mu_n, \nu_n) + d_{\F}(\nu_n, \nu) \quad\quad \text{triangle inequality}\\
		& \leq d_{\F_D}(\mu_n, \nu_n) + (1+ \| \nu_n \|_1) \cdot \max_{f \in \F} \min_{f' \in \F_D} \| f - f' \|_{\infty} + d_{\F}(\nu_n, \nu) \\
		&\leq d_{\F_D}(\mu, \nu_n) + (1+ \| \nu_n \|_1) \cdot \max_{f \in \F} \min_{f' \in \F_D} \| f - f' \|_{\infty} + d_{\F}(\nu_n, \nu) \quad\quad \text{optimality of GAN} \\
		&= \left[d_{\F_D}(\mu, \nu) + d_{\F_D}(\nu, \nu_n) \right] + (1+ \| \nu_n \|_1) \cdot \max_{f \in \F} \min_{f' \in \F_D} \| f - f' \|_{\infty} + d_{\F}(\nu_n, \nu) \\
		&\leq \min_{\mu \in \mu_G} d_{\F_D}(\mu, \nu) + (1+ \| \nu_n\|_1) \cdot \max_{f \in \F} \min_{f' \in \F_D} \| f - f' \|_{\infty} + d_{\F_D}(\nu, \nu_n) + d_{\F}(\nu_n, \nu),
	\end{align*}
	where we use the following fact $d_{\F}(\mu_1, \mu_2) + d_{\F}(\mu_2, \mu_3) \geq d_{\F}(\mu_1, \mu_3)$ and
	\begin{align*}
		d_{\F}(\mu, \nu_n) &= \max_{f \in \F} \int f (d \mu(x) - d\nu_n (x)) \\
		& \leq \max_{f' \in \F_{D}} \int f' (d \mu(x) - d\nu_n (x)) + \max_{f \in \F} \inf_{f' \in \F_D} \|f - f' \|_{\infty} \cdot  \|\mu - \nu_n\|_1\\
		&= d_{\F_D}(\mu, \nu_n) + \max_{f \in \F} \inf_{f' \in \F_D} \|f - f' \|_{\infty} \cdot (1 + \| \nu_n\|_1)
	\end{align*}
	It is easy to bound the following using similar logic
	\begin{align*}
		d_{\F_D}(\mu_n, \nu) & \leq d_{\F_D}(\mu_n, \nu_n) +  d_{\F_D}(\nu_n, \nu)  \\
		& \leq  d_{\F_D}(\mu, \nu_n) +  d_{\F_D}(\nu_n, \nu) \\
		& \leq \min_{\mu \in \mu_G} d_{\F_D}(\mu, \nu) +  d_{\F_D}(\nu, \nu_n) + d_{\F_D}(\nu_n, \nu).
	\end{align*}
\end{proof}

\begin{proof}[Proof of Corollary~\ref{coro:arb}] \quad
	 If $\F_D$ consists of $L$-Lipschitz functions (Wasserstein GAN) on $\mathbb{R}^d$, $d\geq 2$, plug in the $\ell_\infty$-covering number bound for Lipschitz functions,
		\begin{align*}
			&\log \mathcal{N}(\epsilon, \F_{D}, \| \cdot \|_{\infty}) \leq C \left( \frac{L}{\epsilon} \right)^d, \\
			&\E d_{\F_D}(\nu, \hat{\nu}_n) \leq 2 \inf_{0<\delta<1/2} \left( 4\delta + \frac{8 \sqrt{2}}{\sqrt{n}} \int_{\delta}^{1/2}  \sqrt{\log \mathcal{N}(\epsilon, \F_{D}, \| \cdot \|_{\infty}) } d \epsilon \right) \leq 16 \left( \frac{4\sqrt{2C}}{d-2}\right)^{\frac{2}{d}} L n^{-\frac{1}{d}}, \\
			&\E d_{\F_D}(\nu, \hat{\nu}_n) \leq \mathcal{O}\left( \left(\frac{C}{d^2 n} \right)^{-\frac{1}{d}} \right).
		\end{align*}
		This matches the best known bound as in \cite{canas2012learning} (Section 2.1.1).

	Let's consider when $\F_D$ denotes Sobolev space $W^{\beta, 2}$ on $\mathbb{R}^d$. Recall the entropy number estimate for $W^{\beta, 2}$ \citep{nickl2007bracketing}, we have
		\begin{align*}
			\label{eq:symm-sobolev}
			&\log \mathcal{N}(\epsilon, \F_{D}, \| \cdot \|_{\infty}) \leq C \left( \frac{1}{\epsilon} \right)^{\frac{d}{\beta} \vee 2},\\
			&\E d_{\F_D}(\nu, \hat{\nu}_n) \leq \mathcal{O}\left( n^{-\frac{\beta}{d}} + \frac{\log n}{\sqrt{n}} \right).
		\end{align*}
		Remark in addition that the parametric rate $\frac{1}{\sqrt{n}}$ is inevitable, which can be easily seen from the Sudakov minoration,
		\begin{align*}
			\E_{\epsilon} \sup_{f\in \F_D} \frac{1}{n} \sum_{i=1}^n \epsilon_i f(X_i) \geq \sup_{\epsilon} \frac{\epsilon}{2} \sqrt{\frac{\log \mathcal{M}(\epsilon, \mathcal{F}_D, \| \cdot\|_n)}{n}} \geq \frac{1}{\sqrt{n}}.
		\end{align*}

\end{proof}

\begin{proof}[Proof of Theorem~\ref{thm:non-param} and Corollary~\ref{coro:non-param}] \quad
	Denote $g(x):= d\nu/d\pi$ as the density ratio, let's construct the following smoothed version $\tilde{\nu}_n$, with a cut-off parameter $M$ to be determined later. Define the regularized/smoothed density estimate
	\begin{align*}
		\tilde{g}_n(x): = \frac{d \tilde{\nu}_n}{d \pi} = \sum_{\xi \in \mathbb{N}^d}  \tilde{\theta}_{\xi}(g) \prod_{i=1}^d \psi_{\xi_i}(x_i),
	\end{align*}
	where based on i.i.d. samples $X^{(1)}, X^{(2)}, \ldots X^{(n)} \sim \nu$
	\begin{equation*}
		\tilde{\theta}_{\xi}(g) :=
		\begin{cases}
		\frac{1}{n} \sum_{j=1}^n \prod_{i=1}^d \psi_{\xi_i}(X^{(j)}_i), & \text{for $\xi$ satisfies}~ \|\xi\|_{\infty} \leq M \\
		0, &\text{otherwise}
		\end{cases}.
	\end{equation*}
	In other words, $\tilde{\nu}_n$ filters out all the high frequency (less smooth) components, when the multi-index $\xi$ has largest coordinate larger than $M$. Now, for any $f \in \F_D$, write the Fourier coefficients of
	$$
	f(x) = \sum_{\xi \in \mathbb{N}^d} \theta_{\xi}(f) \prod_{i=1}^d \psi_{\xi_i}(x_i).
	$$
	For any $\nu \in \F_R$ (or equivalently $g \in W^{\alpha,2}(L_\alpha)$), we have
	\begin{align*}
		\E d_{\F_D}(\nu, \tilde{\nu}_n) &= \E \sup_{f \in \F_D} \langle f, g - \tilde{g}_n \rangle_\pi \\
		&= \E \sup_{f \in \F_D} \sum_{\xi \in \mathbb{N}^d} \theta_{\xi}(f) \left( \tilde{\theta}_{\xi}(g) - \theta_{\xi}(g) \right) \\
		&= \E \sup_{f \in \F_D} \left\{ \sum_{\xi \in [M]^d} \theta_{\xi}(f) \left( \tilde{\theta}_{\xi}(g) - \theta_{\xi}(g) \right) + \sum_{\xi \in \mathbb{N}^d \backslash [M]^d} \theta_{\xi}(f) \theta_{\xi}(g) \right\} \\
		&\leq \E \sup_{f \in \F_D}  \sum_{\xi \in [M]^d} \theta_{\xi}(f) \left( \tilde{\theta}_{\xi}(g) - \theta_{\xi}(g) \right) + \E \sup_{f \in \F_D} \sum_{\xi \in \mathbb{N}^d \backslash [M]^d} \theta_{\xi}(f) \theta_{\xi}(g).
	\end{align*}
	For the first term,
	\begin{align*}
		& \E \sup_{f \in \F_D}  \sum_{\xi \in [M]^d} \theta_{\xi}(f) \left( \tilde{\theta}_{\xi}(g) - \theta_{\xi}(g) \right) \\
		& \leq \E \sup_{f \in \F_D} \left\{\sum_{\xi \in [M]^d} \theta^2_{\xi}(f) \right\}^{1/2} \cdot \left\{\sum_{\xi \in [M]^d} \left( \tilde{\theta}_{\xi}(g) - \theta_{\xi}(g) \right)^2 \right\}^{1/2} \\
		& \leq L_\beta \E \left\{\sum_{\xi \in [M]^d} \left( \tilde{\theta}_{\xi}(g) - \theta_{\xi}(g) \right)^2 \right\}^{1/2} \\
		& \leq L_\beta \left\{ \sum_{\xi \in [M]^d}  \E \left( \tilde{\theta}_{\xi}(g) - \theta_{\xi}(g) \right)^2 \right\}^{1/2} \quad \text{Jensen's inequality} \\
		& \leq L_\beta \sqrt{C \frac{M^d}{n} } \quad \text{define $C:=\sup_{\xi} \E_{X \sim \nu} \psi^2_{\xi}(X)$.}
	\end{align*}
	as we know
	$$
	\sup_{f \in \F_D} \sum_{\xi \in [M]^d} \theta^2_{\xi}(f) \leq \sum_{\xi \in [M]^d} ( 1+\|\xi\|_2^2 )^{\beta} \theta^2_{\xi}(f) \leq L_\beta^2.
	$$
	Note for trigonometric series $\sup_{\xi} \E_{X \sim \nu} \psi^2_{\xi}(X) \leq 1$.

	For the second term, the following inequality holds
	\begin{align*}
		&\E \sup_{f \in \F_D} \sum_{\xi \in \mathbb{N}^d \backslash [M]^d} \theta_{\xi}(f) \theta_{\xi}(g)\\
		&\leq \E \sup_{f \in \F_D} \left\{\sum_{\xi \in [M]^d} \theta^2_{\xi}(f) \right\}^{1/2} \cdot \left\{\sum_{\xi \in [M]^d} \theta^2_{\xi}(g)  \right\}^{1/2} \\
		&\leq \sup_{f \in \F_D} \left\{ (1+M^2)^{-\beta} \sum_{\xi \in [M]^d} (1+ \| \xi\|_2^2)^{\beta}\theta^2_{\xi}(f) \right\}^{1/2} \left\{ (1+M^2)^{-\alpha} \sum_{\xi \in [M]^d} (1+ \| \xi\|_2^2)^{\alpha} \theta^2_{\xi}(g) \right\}^{1/2} \\
		&\leq L_\alpha L_\beta \sqrt{\frac{1}{M^{2(\alpha+\beta)}}}.
	\end{align*}
	Combining two terms, we have for any $\nu \in \F_R$ (or equivalently $g \in W^{\alpha,2}(L)$),
	\begin{align}
		\sup_{\nu\in \F_R} \E d_{\F_D}(\nu, \tilde{\nu}_n) &\leq  \inf_{M \in \mathbb{N}} \left\{ L_\beta \sqrt{C \frac{M^d}{n} } + L_\beta L_\alpha \sqrt{\frac{1}{M^{2(\alpha+\beta)}}} \right\} \label{eq:l_a_l_b}\\
		&\precsim n^{-\frac{\alpha+\beta}{2(\alpha+\beta)+d}}. \nonumber
	\end{align}
	And the optimal choice of $M \asymp n^{\frac{1}{2(\alpha+\beta)+d}}$. Note that this is a more aggressive smoothing scheme than classic nonparametric estimation with smoothness $\alpha$, due to the fact that we are utilizing the smoothness in the evaluation metric at the same time.

\end{proof}

\begin{proof}[Proof of Theorem~\ref{thm:lower-bound}] \quad
The proof uses the standard Fano's inequality.
	\begin{lem}[Thm. 2.5 in \cite{tsybakov2009introduction}]
		Assume that $H \geq 2$ and suppose $\Theta$ contains $\theta_0, \theta_1, \ldots, \theta_H$ such that:
		\begin{enumerate}
			\item $d(\theta_j, \theta_k) \geq 2s >0$, for all $j, k \in [H]$ and $j \neq k$.
			\item $\frac{1}{H} \sum_{j=1}^H D_{\rm KL}(P_{j}, P_0) \leq \alpha \log H$ with $0< \alpha < 1/8$ and $P_j = P_{\theta_j}$ for $j \in [H]$.
		\end{enumerate}
		Then for any estimator $\hat{\theta}$,
		\begin{align*}
			\sup_{\theta \in \Theta} P_\theta(d(\hat{\theta}, \theta) \geq s) \geq \frac{\sqrt{H}}{1+\sqrt{H}} \left( 1 - 2\alpha - \sqrt{\frac{2\alpha}{\log H}} \right) > 0.
		\end{align*}
	\end{lem}

	Let's construct a mixture of hypothesis on the function space $W^{\alpha, 2}(L)$, and a subset of discriminator functions in $W^{\beta, 2}(L)$, such that the multiple testing problem among the mixture of hypothesis is hard, and thus the loss induced by the best discriminator among the subset provides a lower bound on the estimation rate. Let's construct this mixture in the frequency domain. Choose $M>0$ to be determined later, denote the hypothesis class of interest to be
	\begin{align}
		\Omega_\alpha = \left\{ g_w(x) = \sum_{\xi \in [M]^d} c_{\alpha} w_\xi \psi_\xi(x), w \in \{0,1\}^{M^d} \right\},
	\end{align}
	where
	$$
	c_{\alpha} = \frac{L}{M^{d/2}(1+dM^2)^{\alpha/2}}.
	$$
	It is easy to verify that
	$\Omega_\alpha \subset W^{\alpha,2}(L)$
	because for any $g_w(x)$, we have
	\begin{align*}
		\sum_{\xi \in [M]^d} (1 + \| \xi\|_2^2 )^{\alpha} c_{\alpha}^2 \leq L^2.
	\end{align*}
	Similarly, let's consider $\Lambda_\beta \subset W^{\beta,2}(L)$
	\begin{align}
		\Lambda_\beta = \left\{ f_v(x) = \sum_{\xi \in [M]^d} c_{\beta} v_\xi \psi_\xi(x), v \in \{-1,1\}^{M^d} \right\},
	\end{align}
	where
	$$
	c_{\beta} = \frac{L}{M^{d/2}(1+dM^2)^{\beta/2}}.
	$$

	Take any $w, w'$, we know that
	\begin{align*}
		d(g_w, g_{w'}) &:= \sup_{f \in \F_D} \langle f, g_w \rangle_\pi - \langle f, g_{w'} \rangle_\pi \geq \sup_{f \in \Lambda_{\beta}} \langle f, g_w - g_{w'} \rangle \\
		& = \sup_{v} c_{\alpha} c_{\beta} \sum_{\xi\in [M]^d} v_{\xi} (w_{\xi} - w_{\xi}') \\
		& =  c_{\alpha} c_{\beta} \sum_{\xi \in [M]^d} I(w_\xi \neq w_{\xi}') =  c_{\alpha} c_{\beta} \rho(w, w'),
	\end{align*}
	where $\rho(w, w')$ denotes the Hamming distance between $w$ and $w'$ on the hypercube $\{0,1 \}^{M^d}$. Now we need to construct a subset over the hypercube such that for any pairs $w, w'$, they are separated in terms of Hamming distance.

	The Varshamov-Gilbert bound (Lemma 2.9 in \cite{tsybakov2009introduction}) does the job. We know that there exist a subset $\{ w^{(0)}, \ldots, w^{(H)} \} \subset \{0,1\}^{h}$ such that $w^{(0)} = (0, \ldots, 0)$,
	\begin{align*}
		&\rho(w^{(j)}, w^{(k)}) \geq \frac{h}{8}, ~\forall~j, k\in [H],~ j\neq k, \\
		&\log H \geq \frac{h}{8} \log 2.
	\end{align*}
	In our case $h = M^d$.

	Now let's calculate the probability distance $P_j, P_0$ induced by hypothesis $w^{(j)}$ and $w^{(0)}$, for all $j \in [H]$ to show that information theoretically, it is hard to distinguish the mixture of hypothesis. The following holds,
	\begin{align*}
		D_{\rm KL}(P_j, P_0) &=  \frac{n}{2} c_\alpha^2 \sum_{\xi \in [M]^d} \left( w^{(j)}_{\xi} - w^{(0)}_{\xi} \right)^2 \\
		& \leq \frac{n}{2} c_\alpha^2 M^d.
	\end{align*}
	If we choose $M = \left( \frac{4L^2}{c \log 2} \frac{1}{d^\alpha} \right)^{\frac{1}{2\alpha+d}} \cdot n^{\frac{1}{2\alpha+d}} $, we know
	\begin{align*}
		\frac{1}{H} \sum_{j=1}^H D_{\rm KL}(P_j, P_0) \leq c \log H, ~\text{for some constant $c>0$.}
	\end{align*}
	In this case, for any $j \neq k \in [H]$
	\begin{align*}
		d(g_{w^{(j)}}, g_{w^{(k)}}) = c_\alpha c_\beta \rho(w^{(j)}, w^{(k)}) \geq \frac{M^d}{8} c_\alpha c_\beta = \frac{L^2}{(1+dM^2)^{\frac{\alpha+\beta}{2}}} \asymp  n^{-\frac{\alpha+\beta}{2\alpha +d}}.
	\end{align*}

	Therefore using Fano's Lemma, we reach
	{\small
	\begin{align*}
		\inf_{\hat{\nu}} \sup_{\nu \in \F_R} \E d_{\F_D}( \hat{\nu}, \nu) &= \inf_{\hat{g}} \sup_{g \in W^{\alpha,2}(L)} \E \sup_{f \in W^{\beta,2}(L)} \langle f, \hat{g} - g \rangle_\pi   \\
		&\geq \inf_{\hat{w}} \sup_{w \in \{w^{(0)},\ldots, w^{(H)}\}} \E d(g_{\hat{w}}, g_{w})\\
		&\geq \frac{L^2}{(1+dM^2)^{\frac{\alpha+\beta}{2}}} \cdot \inf_{\hat{w}} \sup_{w \in \{w^{(0)},\ldots, w^{(H)}\}} P\left(d(g_{\hat{w}}, g_{w}) \geq \frac{L^2}{(1+dM^2)^{\frac{\alpha+\beta}{2}}} \right) \succsim n^{-\frac{\alpha+\beta}{2\alpha +d}}.
	\end{align*}}

\end{proof}

\begin{proof}[Proof of Theorem~\ref{thm:deep-relu-learn}] \quad
	We outline the key steps of the proof. Due to the approximation property of the network class $\F_D$ to $W^{\beta, \infty}(1)$, and $\mu_G$ to $W^{\alpha, \infty}(1)$, we know that
	\begin{align*}
		\max_{f \in W^{\beta, \infty}(1)} \min_{f' \in \F_D} \| f - f' \|_\infty \leq \epsilon, \\
		\min_{\mu \in \mu_G} d_{\F_D}(\mu, \nu) = \sup_{f \in \F_D} \| f\|_{1} \| \mu - \nu \|_{\infty} \leq M \epsilon
	\end{align*}
	because $F_D \subseteq W^{0, \infty}(M)$.
	Apply the oracle inequality \eqref{eq:oracle-2}, one has
	\begin{align*}
		d_{W^{\beta, \infty}(1)}(\tilde{\mu}_n, \nu) &\leq \min_{\mu \in \mu_G} d_{\F_D}(\mu, \nu) + (1+ \| \tilde{\nu}_n \|_1) \cdot \max_{f \in W^{\beta, \infty}(1)} \min_{f' \in \F_D} \| f - f' \|_{\infty} \\
		& \quad \quad + d_{\F_D}(\nu, \hat{\nu}_n) + d_{W^{\beta, \infty}(1)}(\tilde{\nu}_n, \nu), \\
		d_{W^{\beta, \infty}(1)}(\tilde{\mu}_n, \nu) - d_{\F_D}(\nu, \tilde{\nu}_n) & \leq (M+2)\epsilon + d_{W^{\beta, \infty}(1)}(\tilde{\nu}_n, \nu).
	\end{align*}
	Here the last step we use the following fact
	\begin{align*}
		\| \tilde{\nu}_n\|_1 \leq  \| \tilde{\nu}_n - \nu \|_1 + 1 \leq \| \tilde{\nu}_n - \nu \|_2  + 1   \leq 1 + L\sqrt{C} n^{-\frac{\alpha}{2(\alpha+\beta)+d}} \leq 2.
	\end{align*}

	It is clear that for $k \in \mathbb{N}$, $W^{k, \infty}(1)$ is a subspace of $W^{k, 2}(r_k)$ under some scaling $r_k$, in the following sense that
	\begin{align*}
		r_k := \sqrt{1+k d^k}, ~~W^{k, \infty}(1) \subset W^{k, 2}(r_k).
	\end{align*}
	Therefore for any $\nu \in W^{\alpha, \infty}(1)$,
	\begin{align*}
		& \E d_{W^{\beta, \infty}(1)}(\tilde{\nu}_n, \nu) \leq \E d_{W^{\beta, 2}(r_\beta)}(\tilde{\nu}_n, \nu) \\
		&\leq \sup_{\nu \in W^{\alpha, 2}(r_\alpha)} \E d_{W^{\beta, 2}(r_\beta)}(\tilde{\nu}_n, \nu) \precsim \min\left\{ n^{ - \frac{\alpha+\beta}{2(\alpha+\beta)+d}}, n^{-\frac{\beta}{d}} + n^{-\frac{1}{2}}  \right\} =: n^{-c(\alpha, \beta, d)}
	\end{align*}
	where the last line follows from Theorem~\ref{thm:non-param}. Combine with the approximation error, we reach
	\begin{align*}
		\E d_{W^{\beta, \infty}(1)}(\tilde{\mu}_n, \nu)- \E d_{\F_D}(\tilde{\nu}_n, \nu) \leq (M+2)\epsilon + C n^{-c(\alpha, \beta, d)}.
	\end{align*}
	Under the assumption that $\F_D \subseteq W^{\beta, 2}(L)$, then
	\begin{align*}
		\E d_{\F_D}(\tilde{\nu}_n, \nu) \leq \E d_{W^{\beta, 2}(L)}(\tilde{\nu}_n, \nu) \precsim n^{-c(\alpha, \beta, d)}.
	\end{align*}
	Putting things together, we have proved $\E d_{W^{\beta, \infty}(1)}(\tilde{\mu}_n, \nu) \leq C_1 \epsilon + C_2 n^{-c(\alpha, \beta, d)}.$

	For the lower bound, the construction of the mixture is slightly different to the construction in Theorem~\ref{thm:lower-bound} (this time the construction is in the spacial domain, and for the density estimation problem). Another difference here lies we need to construct hypothesis in $W^{\alpha, \infty}(1)$ (also called the H\"{o}lder space) instead of $W^{\alpha, 2}(L)$. Our construction is motivated from \cite{tsybakov2009introduction} Page 103 and Theorem 2.8, but note the difference here is that we need to make sure each hypothesis is indeed a valid density function in $W^{\alpha, \infty}(1)$ in $\mathbb{R}^d$.
	Choose a kernel function $K(x) = a \exp(-\frac{1}{1 - 4x^2}) I(|x|<1/2)$ for some small $a>0$ to ensure that $K(x) \in W^{\alpha \vee \beta, \infty}(1)$. Let $m$ be a parameter to be determined later, and denote $h_n = 1/m$.

	Define the hypothesis class to be
	\begin{align*}
		\Omega_\alpha = \left\{ g_w(x) = \frac{1}{1+c_w}\left( 1 + \sum_{\xi \in [m]^d} w_\xi h_n^{\alpha} \varphi_\xi(x) \right), w \in \{0,1\}^{m^d} \right\} , \\
		\Lambda_\beta = \left\{ f_v(x) = \sum_{\xi \in [m]^d} v_\xi h_n^{\beta} \varphi_\xi(x), v \in \{-1,1\}^{m^d} \right\},
	\end{align*}
	where
	\begin{align*}
 	   \varphi_\xi(x) &=  \prod_{i=1}^d K\left(\frac{x_i - \frac{\xi_i - 1/2}{m}}{h_n} \right), ~~\text{$h_n = 1/m$},\\
	   c_w &= \int \sum_{\xi \in [m]^d} w_\xi h_n^{\alpha} \varphi_\xi(x)  dx.
	\end{align*}
	Let's bound $c_w$, for any $w$
	\begin{align*}
		0\leq c_w = \sum_{\xi \in [m]^d} \int  w_\xi h_n^{\alpha} \varphi_\xi(x)  dx \leq \sum_{\xi \in [m]^d} \int  h_n^{\alpha + d} \prod_{i=1}^d K\left(u_i \right) du \leq m^d h_n^{\alpha + d} a^d \leq a^d h_n^{\alpha} \ll 1/100,
	\end{align*}
	and $\| (g_w(x) - g_{0}(x))/g_w(x)\|_{\infty} \leq c_w + a^d h_n^{\alpha} \ll 1/50$.
	It is easy to verify (1) $\Omega_\alpha \subset W^{\alpha, \infty}(L)$ for some $L$, and each element in the set is a valid density; (2) $\Lambda_\beta \subset W^{\beta, \infty}(1)$, as for any multi-index $\gamma$,
	\begin{align*}
		\| D^{(\gamma)} g_w(x) \|_{\infty} \precsim h_n^{\alpha - |\gamma|}, ~~ \| D^{(\gamma)} f_v(x) \|_{\infty} \precsim h_n^{\beta - |\gamma|}.
	\end{align*}

	To apply the Fano's Lemma, let's first show that the hypothesis are separated in the evaluation metric. Let's use the same subset $\{ w^{(0)}, \ldots, w^{(H)} \} \subset \{0,1\}^{h}$ claimed by the Varshamov-Gilbert bound
	\begin{align*}
		&\rho(w^{(j)}, w^{(k)}) \geq \frac{h}{8}, ~\forall~j, k\in [H],~ j\neq k, \\
		&\log H \geq \frac{h}{8} \log 2.
	\end{align*}
	In our case $h = m^d$.
	For the loss function, any $w, w' \in \{ w^{0}, \ldots, w^{H} \}$
	\begin{align*}
		d(g_w, g_{w'}) &:= \sup_{f \in W^{\beta, \infty}(1)} \langle f, g_w \rangle - \langle f, g_{w'} \rangle \geq \sup_{f \in \Lambda_{\beta}} \langle f, g_w - g_{w'} \rangle \\
		& = \sup_{v} h_n^{\alpha +\beta} \sum_{\xi\in [m]^d} v_{\xi} (w_{\xi} - w_{\xi}') \int  \varphi_\xi^2(x)  dx \\
		& =  h_n^{\alpha +\beta + d} \sum_{\xi \in [m]^d} I(w_\xi \neq w_{\xi}') \int \prod_{i=1}^d K^2\left(u_i \right) du  \geq c h_n^{\alpha +\beta + d} m^d \asymp h_n^{\alpha +\beta}.
	\end{align*}

	Now let's show that based $n$ i.i.d. data generated from density $g_w(x)$, it is hard to distinguish the hypothesis.
	Note that for $|x| < 1/50$, we know that $\log(1+x) \geq x - x^2$. Therefore
	\begin{align*}
		D_{\rm KL}(P_{w^{(j)}}^{\otimes n}, P_{w^{(0)}}^{\otimes n}) & = n D_{\rm KL}(P_{w^{(j)}}, P_{w^{(0)}}) \\
		& = n \int  - \log \left( 1 + \frac{g_{0} - g_{w^{(j)}}}{g_{w^{(j)}}} \right) g_{w^{(j)}} dx \\
		& \leq n \int \frac{[g_{0} - g_{w^{(j)}}]^2}{g_{w^{(j)}}} dx \leq 1.01 n \sum_{\xi \in [m]^d} \int  h_n^{2\alpha} \varphi_\xi^2(x)  dx \\
		& \leq 1.01 n \sum_{\xi \in [m]^d} \int  h_n^{2\alpha+d} \prod_{i=1}^d K^2\left(u_i \right) du
		 \precsim n h_n^{2\alpha +d} m^d.
	\end{align*}
	Therefore if we choose $m \asymp n^{-\frac{1}{2\alpha + d}}$, we know
	$\frac{1}{H} \sum_{j=1}^H D_{\rm KL}(P_{w^{(j)}}^{\otimes n}, P_{w^{(0)}}^{\otimes n}) \leq c \log H = c' m^d$.
	Using the Fano's Lemma, the lower bound for density estimation is of the order $h_n^{\alpha+\beta} = n^{-\frac{\alpha+\beta}{2\alpha+d}}$, as
	\begin{align*}
		\inf_{\nu_n} \sup_{d\nu/dx \in W^{\alpha, \infty}(L)}\E d_{W^{\beta, \infty}(1)}(\nu_n, \nu) &= \inf_{\hat{g}} \sup_{g \in W^{\alpha,\infty}(L)} \E \sup_{f \in W^{\beta,\infty}(1)} \langle f, \hat{g} - g \rangle   \\
		&\geq \inf_{\hat{w}} \sup_{w \in \{w^{(0)},\ldots, w^{(H)}\}} \E d(g_{\hat{w}}, g_{w})\\
		&\geq C_3 \left( n^{-\frac{\alpha+\beta}{2\alpha + d}} + n^{-\frac{1}{2}}\right).
	\end{align*}
\end{proof}

\begin{proof}[Proof of Theorem~\ref{thm:relu-discrimin}] \quad
Claim that for any $f(x) \in \F^{(\ell)}(V)$, $f$ is $V^{\ell}$-Lipschitz with respect to the $\ell_\infty$-metric.
		Prove through induction. When $\ell = 1$, $| f(x) - f(y) | = |\sum_j w_j (x_j - y_j)| \leq \| w\|_1 \|x - y \|_{\infty} \leq V\|x - y \|_{\infty}$. Assume for general $\ell$, the claim is true, let's prove for $f \in \F^{(\ell+1)}(V)$
		\begin{align*}
			|f(x) - f(y)| &=  \left| \sum_j w_j \left[ \sigma (f_j(x)) - \sigma(f_j(y)) \right] \right| \\
			&\leq \| w\|_1 \max_j \| \sigma (f_j(x)) - \sigma(f_j(y))  \| \\
			& \leq V \max_j | f_j(x) - f_j(y) | \leq V \cdot V^{\ell} \| x- y \|_{\infty} = V^{\ell+1} \| x- y \|_{\infty}.
		\end{align*}

	Observe that
		$\forall f \in \F^{\ell}(V)$, the following inequality holds
		\begin{align*}
			\max_{|\alpha| \leq 1}\| D^{\alpha} f \|_{\infty} \leq V^{\ell}, \\
			\left( \sum_{|\alpha| \leq 1} \| D^{\alpha} f \|_2^2 \right)^{1/2} \leq (d+1)^{\frac{1}{2}} V^{\ell}.
		\end{align*}
		In other words, $f \in \F_D \subset W^{1, 2}\left( \sqrt{d+1} V^{\ell}  \right)$.

	Now we know that $d\nu(x)/dx \in W^{\alpha,2}(L)$ and $f \in \F_D \subset W^{1, 2}\left( \sqrt{d+1} V^{\ell}  \right)$. Due to the oracle inequality \ref{thm:oracle}, we know
	\begin{align*}
		\E d_{\F_D}(\tilde{\mu}_n, \nu) -  \min_{\mu \in \mu_G} d_{\F_D}(\mu, \nu) \leq 2 \E d_{\F_D} (\nu, \tilde{\nu}_n) \leq 2 \E d_{W^{1, 2}\left( \sqrt{d+1} V^{\ell}  \right)} (\nu, \tilde{\nu}_n).
	\end{align*}
	Apply Theorem~\ref{thm:non-param} (Eqn.~\eqref{eq:l_a_l_b}) with $\beta = 1$, $L_\alpha = L$, $L_\beta = \sqrt{d+1} V^{\ell}$, we can complete the proof.
\end{proof}

\begin{lem}[Relative Ordering]
	\label{lem:rela.order}
	Consider the following two function spaces on $\mathbb{R}^d$
	\begin{align*}
		\F_{\rm Lip,\ell_\infty}(L, M) := \{ f: |f(x) - f(y)| \leq L \|x - y \|_{\infty},~\forall x, y, ~\text{and}~ \| f\|_{\infty} \leq M \},\\
		\F_{\rm Lip, \ell_2}(L, M) : = \{ f: |f(x) - f(y)| \leq L \|x - y \|_{2},~ \forall x, y,~\text{and}~ \| f\|_{\infty} \leq M  \}.
	\end{align*}
	The following holds true
	\begin{align*}
		\F_{\rm Lip,\ell_\infty}(L, M) \subseteq  \F_{\rm Lip,\ell_2}(L, M) \subseteq W^{1, \infty}(L \vee M),\\
		\F_{\rm Lip,\ell_\infty}(dL, L) \supseteq  \F_{\rm Lip,\ell_2}(\sqrt{d}L, L) \supseteq W^{1, \infty}(L).
	\end{align*}
\end{lem}
\begin{proof} \quad
	Without loss of generality assume the function is differentiable. Then
	$$W^{1, \infty}(L) = \{f: |f(x) - f(y)| \leq L \|x - y \|_{1},~\forall x, y, ~\text{and}~ \| f\|_{\infty} \leq L \}.$$
	Observe that $\|x - y \|_1 \geq \| x- y \|_{2} \geq \| x- y \|_{\infty}$. Then $|f(x) - f(y)| \leq L \|x - y \|_{\infty} \leq L \|x - y \|_{2} \leq L \| x - y\|_1$.
	For the other side of the inequality, $\|x - y \|_1 \leq \sqrt{d} \| x- y \|_{2} \leq d \| x- y \|_{\infty}$, hence
	$|f(x) - f(y)| \leq L \|x - y \|_{1} \leq L\sqrt{d} \|x - y \|_{2} \leq Ld \| x - y\|_\infty$.
\end{proof}

\section*{Acknowledgement}
The author thanks Sasha Rakhlin for helpful discussions.

\bibliographystyle{plainnat}
\bibliography{ref}

\end{document}